%% file: Regression_any_loss.tex
\documentclass[10pt, a4paper]{article}

\input{packages_notes}

\input{commandes_guillaume}

\usepackage{enumitem}

\newcommand{\PP}{\mathbb{P}}
\usepackage[]{}
\usepackage{subcaption}

\DeclareMathOperator*{\VC}{VC}
\DeclareMathOperator*{\Tr}{Tr}

\newcommand{\Ba}{B_\Sigma}
\usepackage{braket}
\usepackage{authblk}

\newtheorem{coro}{Corollary}

\newtheorem{setting}{Setting}
\newtheorem{theo}{Theorem}
\newtheorem{lemma}{Lemma}

\newtheorem*{rem}{Remark}

\DeclareMathOperator*{\polylog}{polylog}

\begin{document}

\title{ A spectral algorithm for robust regression with subgaussian rates}
\author[1]{Jules Depersin \\ email: \href{mailto:jules.depersin@ensae.fr}{jules.depersin@ensae.fr} \\ CREST, ENSAE, IPParis. 5, avenue Henry Le Chatelier, 91120 Palaiseau, France.}

\date{}                     
\setcounter{Maxaffil}{0}
\renewcommand\Affilfont{\itshape\small}

\maketitle

\begin{abstract}

We study a new linear up to quadratic time algorithm for linear regression in the absence of strong assumptions on the underlying distributions of samples, and in the presence of outliers.  The goal is to design a procedure which comes with actual working code that attains the optimal sub-gaussian error bound even though the data have only finite moments (up to $L_4$) and in the presence of possibly adversarial outliers. A polynomial-time solution to this problem has been recently discovered \cite{10.1145/3357713.3384329} but has high runtime due to its use of Sum-of-Square hierarchy programming. At the core of our algorithm is an adaptation of the spectral method introduced in \cite{lei2019fast} for the mean estimation problem to the linear regression problem. As a by-product we established  a connection between the linear regression problem and the furthest hyperplane problem. From a stochastic point of view, in addition to the study of the classical quadratic and multiplier processes we introduce a third empirical process that comes naturally in the study of the statistical properties of the algorithm. We provide an analysis of this latter process using results from \cite{depersin2020robust}. 
\end{abstract}


\section{Introduction} 
\label{sec:introduction_on_the_mean_vector_problem}

Much work concerning the prototypical problem of regression focuses on the study of rates of error of a given statistical procedure while making strong assumptions on the underlying distributions of samples, assuming for instance that they are i.i.d. and subgaussian or bounded (see for instance, \cite{Kol11,MR2319879,lecue2013learning}). It is however of fundamental importance to understand what happens when the data violates such strong assumptions, for instance, when the underlying distribution of samples is \emph{heavy-tailed} and/or when the dataset is corrupted by outliers. In such cases -- which are everyday cases for real-world datasets -- classical estimators such as OLS or MLE  exhibit, at best, far-from-optimal statistical behaviours and at worst completely non-sens outputs. In this work, we study the statistical properties (non-asymptotic estimations and predictions results) of algorithms coming with actual working code constructed on this type of real-word datasets. We want to put forward that it is an algorithm and not only a purely theoretical estimator and that this algorithm can be coded efficiently (we provide a simulation study in the following) since its most time consuming fundamental building block is to find a top singular vector of a reasonable size matrix. However, our theoretical results show that even though the dataset is far from the ideal i.i.d. subgaussian framework and even though we study an actually codable algorithm, the resulting estimator achieves the very same minimax bounds with (exponentially) high probability as the MLE/OLS does in the ideal i.i.d. Gaussian framework (i.e. Gaussian design and independent Gaussian noise), (see \cite{lecue2013learning} for deviation optimal result in the ideal framework). On top of that, we prove a theoretical running time for that algorithm which can be linear $\cO(Nd)$ (where $N$ is the sample size and $d$ is the number of features) and at most quadratic $\cO(N
^2d)$.\\

Robustness has been a classical topic in statistics since the work of Hampel \cite{MR0301858,MR0359096}, Huber(\cite{MR2488795, MR0161415}) and Tukey \cite{MR0133937}. For a statistical problem such as mean estimation, regression or covariance estimation, we are given a loss function and an associated risk function $\ell$ (for instance, for the problem of estimation of the mean vector  $\mu^*:=\E(X)\in\bR
^d$, the loss function is $\ell_\mu(X) =\norm{\mu -X}_2^2, \forall \mu\in\bR^d$ and the associated risk if $\ell(\mu)=\E \ell_\mu(X) $).  For robust estimators the emphasis is not put on the expected risk $\E(\ell(\hat \mu))$ -- where the expectation is taken w.r.t. the data -- but rather on the dependence of the risk bound $r_\delta$ on the confidence level $1-\delta \in [0,1]$: we want to find the smallest $r_\delta$ so that $ \PP(\ell(\hat \mu)> r_\delta)\leq \delta$ and the way $r_\delta$ depends on $\delta$ is paramount in this approach (this is a key property of the estimator $\hat\mu$ that cannot be revealed when its expected risk is studied). An estimator is robust if the rate $r_\delta$ does not grow "too quickly" when $\delta$ goes to $0$: we look for an optimal dependence called "subgaussian rate", because it is the dependency that we would get if all the data were sub-gaussian. It has been known that for the problem of estimating the mean in one dimension ($d=1$) under the only assumptions that the random variables have bounded variance $\sigma$, there are estimators which achieve rates whose dependence on $\delta$  is way better than the empirical mean \cite{AIHPB_2012__48_4_1148_0}. Indeed, while the empirical mean cannot achieve in general a better rate than $\sigma/\sqrt{ N \delta} $ ($N$ being the number of sample), the median of means estimator for instance achieves \emph{the same rate as the empirical mean does in the Gaussian setting} $ \sigma \sqrt{\log(1/\delta)/N}$ (see \cite{AIHPB_2012__48_4_1148_0,MR3576558}). Achieving similar guarantees for large dimensions $d$ is much more difficult, even without asking for computationally-tractable algorithms. However, a number of estimators did succeed, in the last decade, to match the rates achievable in the Gaussian case by usual approaches (sometimes called "subgaussian rates") with much weaker assumptions, even in high dimensions, for problems such as regression or mean estimation (see \cite{lugosi2019sub,LMSL} or \cite{lugosi2019mean} for a survey). Here we consider the standard linear regression setting where data are couples $(X_i, Y_i)_i \in \R^d \times \R$ and we look for the best linear combination of the coordinates of an input vector $X$ to predict the output $Y$, that is we look for $\beta^*$ defined as follows.
$$\beta^*=\argmin_{\beta \in \R^d}\ell(\beta)=\argmin_{\beta \in \R^d} \E(Y_1-\braket{\beta , X_1})^2 .$$ 

The theoretical question of finding robust to heavy-tailed estimators reaching optimal rates for the regression problem has attracted much attention during the last ten years. It first started with the study the standard procedures in this heavy-tailed framework, such as Empirical Risk Minimization or its regularized versions \cite{lecu2016regularization,MR3301300, LeM14,MR3568047}. Several results showed the negative but unavoidable impact of heavy-tailed data on these classical procedures \cite{MR3474824}. In the mean time, new estimators have been introduced.  For instance, the pioneer work of \cite{audibert2011} has considered weak moment conditions, such as a $L_2-L_4$ norm equivalence, under which the subgaussian rate could be reached. It was then followed by a rich literature such as \cite{lugosi2016risk,guillaume2017learning, lecu2017robust, MR3568047, MR3301300}. The remaining issue is that naive methods to compute these new theoretically-optimal estimators take exponential time  in the number of dimension $d$, partly because some of them are based on non-convex optimization. \\

Some recent works (for instance \cite{MR3631028, MR3909639, MR3909640, minsker2015geometric, JMLR:v17:14-273}) focused on providing procedures that were not only robust (to outlier or heavy tailed data), but also computationally efficient. Unfortunately, in \cite{MR3631028, MR3909639, MR3909640}   the procedures fail with constant probability, failing to give a good dependence on $\delta$, and procedures from \cite{minsker2015geometric, JMLR:v17:14-273} do not achieve optimal rates: for instance, for the problem of mean estimation with bounded covariance $\Sigma$, they achieve a bound of order $\propto \sqrt{ \frac{\Tr(\Sigma) \times \log(1/\delta)}{N}}$ (up to constants) when the rates achievable in the Gaussian case by the empirical mean is of order $\sqrt{ \frac{\Tr(\Sigma)+\norm{\Sigma}_{op} \log(1/\delta)}{N}}$. This suggests an important question: are there efficiently computable procedures achieving optimal rates $r_\delta$ under weak assumptions on underlying data, and in the presence of outliers among the data? \\

 This question was recently answered affirmatively for the mean estimation problem. Indeed, recent advances have shown that, for the problem of mean estimation, one could find computationally efficient procedures (that is to say polynomial in both the dimension $d$ and the number of data $N$) that are statistically nearly optimal, meaning that they reach -up to universal constants- the optimal radius $r_\delta=\sqrt{ \frac{\Tr(\Sigma)+\norm{\Sigma}_{op} \log(1/\delta)}{N}}$ for every confidence level $\delta \in [0,1]$ (see \cite{hopkins2018sub, Bartlett19, depersin2019robust}). More recently, \cite{lei2019fast} introduce a spectral method reaching the optimal sub-gaussian rates without using Semi-Definite Programming, making somehow robust mean estimation easier to understand, easier to interpret and easier to code while still keeping optimal statistical results. \\

The question of whether reaching similar bounds (matching the one of the OLS in the Gaussian setting without the Gaussian and i.i.d. assumptions -- thus allowing for corrupted and heavy-tailed datasets) in polynomial time was possible for other statistical problems such as regression or covariance matrix estimation had been open for a long time. Indeed, up to recently, the best known polynomial algorithms were the one from \cite{prasad2018robust} or from \cite{JMLR:v17:14-273}. The guarantee is the same for those two algorithms: when the covariance of $X$ is the identity and when the noise $\xi=Y-\braket{\beta
^*,X}$ has bounded variance  $\ell(\hat f) - \ell(f^*) \leq \cO(\frac{\log(1/\delta) d}{N}) $ with probability $1-\delta$, and they need a number of sample of order $N \gtrsim \log(1/\delta) d $. The article \cite{10.1145/3357713.3384329} has been the first to construct a polynomial-time method achieving the rate of the OLS in the Gaussian setting $\ell(\hat f) - \ell(f^*) \leq \cO(\frac{\log(1/\delta) \vee d}{N}) $. To the date, it is the only procedure running in polynomial algorithm achieving the optimal subgaussian rate. However,  \cite{10.1145/3357713.3384329} uses the Sum of Square (SoS) programming hierarchy  to design their algorithm. Even if SoS hierarchy runs in polynomial time,  its reliance on solving large semi-definite programs makes it impractical and is still a theoretical result leaving still open the question on the existence of a practical efficient algorithm achieving optimal subgaussian rates.   

In this article, we tackle this issue, showing that techniques from \cite{lei2019fast} combined with lemmas from $\cite{depersin2020robust}$ can be used to give the first practical, nearly quadratic (and in fact in most cases nearly-linear) algorithm that reaches the subgaussian rate. We also conduct numerical experiments on simulated data with our proposed procedure to show that it is indeed practical and fast. Moreover, as predicted by our theoretical findings, our simulation analysis shows that it is robust both to heavy-tailed data and to outliers. To the best of our knowledge, this is the first time that numerical experiments are conducted for a regression algorithm with sub-gaussian rates and polynomial time guarantees. \\

From a theoretical point of view, our main result (that we will prove later) can be stated as follows (see Setting~\ref{setting} for the precise set of assumptions and next sections for the construction of the algorithm).

\begin{theo}\label{theo:main1}
There are universal constants $A,B,C$ so that the following hold. Let $\delta \geq e^{-AN}$  and $K \geq B (\lfloor  \log(1/\delta) \rfloor \vee d \vee |\cO|)$ where $|\cO|$ is the number of outliers. Given $N \geq K$ points, there is an algorithm running in time $$\mathcal{O}
\left(( n d + k^2d) \times   \log(||\beta^*||_\Sigma)\times \polylog(k,d) \right)$$ that outputs an estimate $\hat\beta \in \R^d$ such that with probability at least $1-\delta$ 
$$  \ell(\hat \beta)- \ell(\beta^*)\leq C \frac{\sup_{u \in \Ba} \E(\xi_1^2 \braket{u , X_1}^2) K }{N}.$$
\end{theo}

So for $K= B ( \lfloor\log(1/\delta) \rfloor \vee d \vee |\cO| )$, we get, up to universal constants the (deviation minimax optimal) subgaussian rate achieved by OLS in the Gaussian framework (see \cite{lecue2013learning}). This rate was achieved previously under similar assumptions by Median-of-means estimators in \cite{lugosi2016risk, lerasle2019lecture, lecu2017robust,guillaume2017learning}  but none of them come with computational time guarantees.

To construct estimator $\hat \beta$ from Theorem~\ref{theo:main1} and to prove its theoretical properties as stated in Theorem~\ref{theo:main1},  we outline now the role of the following key tools:
\begin{itemize}
    \item \textbf{Median of Means \cite{MR702836, MR855970,MR1688610}}: this approach is nowadays widely used in robust estimation (see \cite{lugosi2019sub, lugosi2019mean, lecu2017robust, chinot2018statistical, minsker2015geometric}, and see \cite{MR3576558} for a good introduction to this technique). Let us quickly explain this trick in one dimension. Consider the problem of estimating the mean $\mu$ of a one-dimensional random variable $X$ from corrupted samples, supposing only $\E((X-\mu)
   ^2 ) \leq \sigma^2$. In that case, the empirical mean fails to provide any guarantees in the presence of outliers, and only gives weak bounds (of order $\sigma\sqrt{1/(\delta n)} $ for the confidence level $1-\delta$, \cite{MR3052407}) even when there are no outliers. 
   
   The median of mean method, that can be traced back to \cite{MR702836} for a confidence level $1-\delta$ consists in splitting the data $X_1, ..., X_N$ into $K\sim \log(1/\delta)$ equal-size buckets. For all $1\leq k\leq K$, $\bar{X}_k$ denotes the average of the samples in bucket $k$. Then we let $\mu_\delta$ be the median of $\bar{X}_1,\cdots, \bar{X}_K$. We can show with a straightforward analysis ($\cite{MR3576558, MR702836}$) that $|\mu_{\delta}-\mu| \leq \sigma \sqrt{\log(1/\delta)/N}$ with probability $\geq 1-\delta$ and that this bound still holds in the presence of up to $K/8$ outliers in the data. The main challenge is to extend this idea to higher dimensional settings and to other statistical problems, where we need to design appropriate notions of median.  For instance, \cite{lugosi2016risk, lerasle2019lecture, lecu2017robust,guillaume2017learning,depersin2020robust} introduce  median-of-mean estimators suited for regression but which are intractable in practice. 
   
    \item The \textbf{Furthest hyperplane problem} was first adapted to compute median-of-mean estimators very recently by \cite{lei2019fast}. Authors from \cite{lei2019fast} adapt to the problem of robust mean estimation a procedure initially proposed by \cite{pmlr-v23-karnin12} to find  the approximate furthest hyperplane, that is to say the hyperplane that separate $0$ from \emph{most of the data} and that is the furthest possible from $0$. The method from \cite{pmlr-v23-karnin12} is based on the multiplicative weight update method (see \cite{v008a006} for a survey), a technique which allows to compute efficiently approximations of quantities such as 
    $\inf_{w_i \in \Delta} \sup_u \sum_i w_i \braket{u, x_i }^2$ where $\Delta$ is a convex set of positive weights.
\end{itemize}
The combination of these two techniques is at the heart of both the construction and the statistical and computational time studies of the algorithm satisfying Theorem~\ref{theo:main1}. 

 In section \ref{stochastic} we present the assumptions we make on the data and provide all the stochastic lemmas that will be needed for the algorithm. In section \ref{section:algo} we will present our descent algorithm and give its precise statistical performance. In section \ref{sec:expe} we present some empirical results on simulated data.

\section{Assumptions and preliminary stochastic results}
\label{stochastic}
\subsection{Assumptions} 
As explained in the previous section, the observed dataset $(\tilde X_i, \tilde Y_i)_{i=1}^N \in \R^d \times \R$ is a corrupted version of the i.i.d. dataset $\{(X_i, Y_i)_i, i\in\{1,\ldots,N\}\}$ in a possibly adversarial way.
The assumptions made on good data $(X_i, Y_i)_i$ are gathered in the following setting: (see also \cite{lerasle2019lecture} or \cite{audibert2011}).
  
 \begin{setting}\label{setting}
We assume that the  following "heavy-tailed setting" holds:
 \begin{enumerate}
     \item $X_1$ has finite second moments; we write its $L^2$-moments matrix $\Sigma=\E(X_1 X_1^T)$ and we assume that $\Sigma$ is known. Let also $\Ba=\{x \in \R^d| \braket{x,\Sigma x} \leq 1\}$ be the ellipsoid associated with this $L_2$ structure and, for $u \in \R^d$ $\norm{u}_\Sigma^2=\braket{u|\Sigma u}$. 
     \item Let $\xi_1=Y_1-\braket{\beta^* , X_1}$ and assume that $\sigma^2 := \sup_{u \in \Ba} \E(\xi_1^2 \braket{u , X_1}^2)$ is such that $\sigma^2 < \infty$.
     \item There exists an universal constant $\gamma$ such that, for all $u\in\bR^d$,  $\gamma \E(\braket{u,X}^2) \geq  \sqrt{\E(\braket{u,X}^4)}$.
 \end{enumerate}
 We assume \emph{adversarial contamination} on the data: $(X_1,Y_1), \cdots,  (X_N, Y_N)$ denote $N$ i.i.d. random vectors in $\mathbb{R}^d\times \R$. The vectors $(X_1,Y_1), \cdots,  (X_N, Y_N)$ are not observed, instead, there exists a (possibly random) set $\mathcal{O}$ such that, for any $i \in \mathcal{O}^c, \ (\tilde X_i, \tilde Y_i) = (X_i, Y_i)$. The set of indices of outliers $\cO$ can be arbitrarily correlated with the data $(X_i, Y_i)$ -- for instance, only the $9N/10$ data with the largest $\norm{X_i}_2$ are observed -- and the outliers $(\tilde X_i, \tilde Y_i)_{i\in\cO}$ can be anything (they can be arbitrarily correlated between themselves and with the non-corrupted data $(X_i, Y_i), i=1, \ldots, N$). The only constraint on $\mathcal{O}$ is on its size: we suppose that we know an upper bound of $|\cO|$ (even though, this constraint may be dropped out if we use an adaptive scheme on $K$ such as Lepski's method in the end). The observed dataset is therefore $\{ (\tilde X_i, \tilde Y_i) : i =1,\cdots,  N \} $, and we want to recover $\beta
^*$ out of it. 
 \end{setting}
 
 Let us now comment on Setting~\ref{setting}. The first three assumptions deal with the heavy-tailed setup. It involves at most the existence of a fourth moment on the noise $\zeta$ and the functions class $\{u\in\bR^d\to\inr{u, X}\}$. The strongest assumption among them is the third one which is a $L_2/L_4$ norm equivalence assumption. This type of assumption has been used from the beginning for the statistical study of ERM and other classical methods in the heavy-tailed scenario for instance in \cite{MR3568047,MR3301300,LeM14} or in \cite{audibert2011}. It is also related to the small ball assumption from \cite{KoM13}. It has been systematically used for the study of Median-of-means estimators (see \cite{lerasle2019lecture}). The remaining of Setting~\ref{setting} deals with the adversarial contamination model. This covers many classical setup such as Huber's $\eps$-contamination model or the $\cO\cup\cI$ framework from \cite{guillaume2017learning,lecu2017robust}. It is somehow the strongest contamination model since an adversary is allowed to modify without any restriction up to $|\cO|$ data before the dataset is revealed to the statistician (for more details, see \cite{pmlr-v20-biggio11} or \cite{MR3909639}).
 



\subsection{Bounds on three stochastic processes} 
 In this section, we introduce three stochastic processes that play a central role in our analysis. We provide a high probability control for the supremum of the three of them into three lemmas. All the stochastic tools that we will need later will be related to one of the three processes. So that all the stochastic part of this work is gather into this section and in the end we will identify a single event onto which the study of the algorithm will be using purely deterministic arguments.
 
 We now state the three lemmas. The two first one deal with the classical quadratic and multiplier processes which already appeared in the study of ERM in \cite{lecue2013learning}. They naturally show up when the quadratic loss is used. The last one is new and is related to the descent algorithm we are studying below.  
 
 We split the data in $K$ blocks that we note $B_k, k \in \{1,\ldots, K\}$, in agreement with the Median-of-Mean framework.  We note $m=N/K$ the number of data in each blocks, and we note $\textbf{X}_k=(X_i)_{i \in B_k}$ and $ \tilde{\textbf{X}}_k=(\tilde X_i)_{i \in B_k}$. $\textbf{Y}_k$ and $ \tilde{\textbf{Y}}_k$ are defined the same way. We start by stating \cite[Lemma 2]{depersin2020robust}, that we will use several times in what follows. We refer the reader to \cite[Definition 1]{depersin2020robust} or \cite{vandervaart} for a definition of the VC-dimension of a set of functions 
  
  \begin{Lemma} \label{main}
  Let $\cF$ be a set of Boolean functions satisfying the following assumptions.

\begin{itemize}
    \item For all $f \in \cF$, $\PP\left(f(\textbf{X}_1, \textbf{Y}_1)=0\right) \geq 31/32$.
    \item $K \geq  C (\VC(\cF) \vee |\cO|)$ where $C$ is a universal constant.
\end{itemize}

 Then, with probability $\geq 1-\exp(-K/512)$, for all $f \in \cF$, there is at least $19K/20$ blocks $B_k$ on which $f(\tilde{\textbf{X}}_k, \tilde{\textbf{Y}}_k)=0$.
  \end{Lemma}
  
  This lemma is used as a baseline to  prove the three following lemmas that will define the three stochastic events $\cA$, $\cB$ and $\cE$ that are needed for our algorithm to give a good estimate. We state in this section that all three fail with exponentially low probability. We introduce the rate
  \begin{equation}
  \label{defr}
       r= 8 \sigma \sqrt{ \frac{K}{N}}.
  \end{equation}
 
\begin{Lemma}[\textbf{Multiplier process}] \label{multiplier}

There is a universal constant $C_1$ so that the following hold. If $K\geq C_1 (d\vee |\cO|)$, the following event $\cE$ has probability $\geq 1 - \exp(-K/512) $ : for all $u \in \Ba$, there exist more than $19/20 K$ blocks $B_k$ so that :
 
  $$ \frac{1}{m}  |\sum_{i \in B_k}(\tilde Y_i-\braket{\beta^* , \tilde X_i})\braket{u ,\tilde X_i}| \leq   r. $$
\end{Lemma}

This can also be also written as: for all $u \in \R^d$ there exist more than $19/20 K$ blocks $B_k$ so that :
 
  $$ \frac{1}{m}  |\sum_{i \in B_k}(\tilde Y_i-\braket{\beta^* , \tilde X_i})\braket{u ,\tilde X_i}| \leq   r ||u||_\Sigma. $$

\begin{Lemma}[\textbf{Quadratic process}]  \label{produitscal}
There is $C_1$ a universal constant so that the following hold. If $K\geq C_1 (d\vee |\cO|)$ the following event $\cB$ has probability probability $\geq 1 - \exp(-K/512) $: for all $u, v \in \R^d$, there exists more than $19/20K$ blocks $B_k$ so that :

 $$|\frac{1}{m}\sum_{i \in B_k} \braket{u ,\tilde X_i}\braket{v ,\tilde X_i}-\braket{u , \Sigma v} | \leq   6 \gamma \sqrt{\frac{1}{m}} \norm{u}_\Sigma \norm{v}_\Sigma$$

In particular, when $m \geq 360 \ 000 \gamma^2 $, on the event $\cB$, for all $u \in R^d$

$$99/100 \braket{u,\Sigma u} \leq \frac{1}{m}\sum_{i \in B_k}\braket{u ,\tilde  X_i}^2 \leq 101/100 \braket{u,\Sigma u} . $$
\end{Lemma}

\begin{Lemma} \label{init}
There is $C_1$ a universal constant so that the following hold. If $K\geq C_1 (d\vee |\cO|)$ and $m \geq 128 \gamma$, then the following event $\cA$ has probability $\exp(-K/512)$. For all $\beta_c \in \R^d$, there exist more than $19/20K$ blocks $B_k$ so that : 

$$ \norm{\tilde Z_k(\beta_c)  }_2 \leq 8\sqrt{\frac{\E(||(\xi_1 \Sigma^{-1/2} X_1||_2^2)}{m}} +\sqrt{d}\norm{\beta_c-\beta^*}_\Sigma \leq \sqrt{d}(r+\norm{\beta_c-\beta^*}_\Sigma) $$

where 
$$\tilde Z_k(\beta_c)=\frac{1}{m} \sum_{i \in B_k}(\tilde Y_i-\beta_c \tilde X_i) \Sigma^{-1/2} \tilde X_i $$
with $r$ defined as in \ref{defr}.
\end{Lemma}

\textbf{We assume for the rest of this work, that $K \geq C_1 (d \vee |\cO| ) $, that $m \geq 360 \ 000 \gamma^2$. We moreover assume that events $\cal A$,  $\cal B$ and $\cal E$ hold.}

\section{Analysis of the algorithm}
\label{section:algo}

The general algorithm, as in \cite{lei2019fast}, is a basic descent procedure :

\vspace{0.7cm}
 \begin{algorithm}[H]\label{algo:descent}
\SetKwInOut{Input}{input}\SetKwInOut{Output}{output}\SetKw{Or}{or}
\SetKw{Return}{Return}
\Input{$\tilde X_1, \tilde Y_1 \ldots, \tilde X_N, \tilde Y_N$, $K \geq C_1 (d \vee |\cO| )$, and $T_{des}$. }
\Output{A robust estimator of $\beta^*$}  
\BlankLine
Initialize $\beta_0= 0$\\
\For{ $t =1, ...,  T_{des} $}{
 $d_t = \texttt{stepSize}(\tilde X,\tilde Y, K,  \beta_t, T_{des} )$ \\
 $g_t = \texttt{descentDirection}(\tilde X,\tilde Y, K,  \beta_t, d_t, T_{des})$ \\
 $\beta_{t+1}= \beta_t- d_t g_t$
}
\Return $\texttt{ROUND}(Z,\theta,(u_t)_t)$.
 \caption{Main descent algorithm}
\end{algorithm}
\vspace{0.7cm}

A good descent direction $v$ should check $\braket{v, \Sigma \beta_t-\beta^*} \geq  c_0 ||\beta_t-\beta^*||_\Sigma$ and $||v||_\Sigma=1$ for some constant $c_0 <1$, and a good step size should check $d_t \in [c_1\norm{\beta_t-\beta^*}_\Sigma, c_0 \norm{\beta_t-\beta^*}_\Sigma] $ with $0<c_1 < c_0$ so that 

$$||\beta_{t+1}- \beta^*||_\Sigma^2 \leq (1-2 c_0c_1 + c_1^2)||\beta_{t}- \beta^*||_\Sigma^2 \leq \alpha ||\beta_{t}- \beta^*||_\Sigma^2$$

with $\alpha < 1$. In order to find a good descent direction, we will be using the central quantity  $$Z_k(\beta_c)=\frac{1}{m} \sum_{i \in B_k}( Y_i-\beta_c  X_i) \Sigma^{-1/2}  X_i$$ already mentioned in the previous section (see Lemma \ref{init}). We decompose $ Z_k$ as $ Z_k(\beta_c)= \frac{1}{m} \sum_{i \in B_k} \xi_i \Sigma^{-1/2}  X_i+ \sum_{i \in B_k} \braket{\beta^*-\beta_c, X_j }\Sigma^{-1/2} X_i$. The first term has mean zero by definition of $\beta
^*$, but the expectation of the second one is $\Sigma^{1/2}(\beta^*-\beta_c)$. So if we find a direction so that most $Z_k(\beta_c)$ are "aligned" with this direction, we might have a shot at finding a descent direction. The introduction of this quantity $Z$ is the main novelty of this work. We will see in the rest of this section that finding such a direction indeed leads to a nice  descent, and we will show how to find it efficiently. \\

More precisely, we will show that the algorithms $\texttt{stepSize} $ and $\texttt{descentDirection} $ are good step size and descent direction. The main tool is a modification of the algorithm \texttt{APPROXBREGMAN} from \cite{lei2019fast} (which is in turn an adaptation from \cite{karnin2011furthest}), that we called \texttt{BregmanRegression}\\

We summarize the properties of this descent in a main theorem :
\begin{theo} \label{theo:main}
On the event $\cE, \cA, \cB$, each iteration of Algorithm \ref{algo:descent} checks the following with probability $\geq 1-\exp(K)/T_{des}$
\begin{itemize}
    \item Whenever $ ||\beta_c-\beta^*||_{\Sigma} \geq 100 r $, 
    $$\norm{\beta_{c+1}-\beta^*}_\Sigma \leq  (1-2/100.000) \norm{\beta_c-\beta^*}_\Sigma $$
    \item Whenever $ ||\beta_c-\beta^*||_{\Sigma} \leq 100 r $, 
    $$ \norm{\beta_{c+1}-\beta^*}_\Sigma \leq 102 r$$
Moreover, each iteration runs in time $\cO((Nd+K^2d) \times \polylog(d,K) )$
\end{itemize}

\end{theo}

To prove this, we need a few intermediate lemma and algorithms. All the results hold on the event $\cA \cap \cB \cap \cE$ We first state some essential remarks about pruning. Because we are on $\cA$, we know that  $9/10K$ blocks check $\norm{\tilde Z_k(\beta_c)  }_2  \leq \sqrt{d}(r+\norm{\beta_c-\beta^*}_\Sigma)$. For simplicity, we will just note $\tilde Z_k(\beta_c) =\tilde Z_i$ We note $K'= \lfloor 9/10K \rfloor$, and we note $Z'_{1}, ..., Z'_{K'}$ the $K'$ smallest $\tilde Z_i$, as returned by algorithm \ref{algo:pruning}. For the rest of this part we will mainly work with the pruned data, so that, on $\cA$, $R := \max_{k\leq K'}|| Z'_k|| <  \sqrt{d}(r+\norm{\beta_c-\beta^*}_\Sigma).$  

\vspace{0.7cm}
 \begin{algorithm}[H]\label{algo:pruning}
\SetKwInOut{Input}{input}\SetKwInOut{Output}{output}\SetKw{Or}{or}
\SetKw{Return}{Return}
\Input{$\tilde Z_1,  ... , \tilde Z_K$ }
\Output{Pruned $\tilde Z_{\sigma(1)}, ..., \tilde Z_{\sigma(K')}$}  
\BlankLine
Compute the norms $||\tilde Z_i||$ and sort them $\tilde Z_{\sigma(1)}< \tilde Z_{\sigma(2)} <... < \tilde Z_{\sigma(K)}$ \\
Remove the top $1/20$\\
\Return $\tilde Z_{\sigma(1)}, ..., \tilde Z_{\sigma(K')} := (Z'_{k})_{k\in\{1, ..., K'\}}$.
 \caption{Pruning algorithm}
\end{algorithm}
\vspace{0.7cm}

Now the first lemma of this part states that if $\cQ^{8/10}$ is the $8/10$ quantile of a serie,  $\max_{u\in \mathcal{B}_2^d}\cQ^{8/10}( \braket{Z'_i, u} )$ is a good estimate of the distance $||\beta_c-\beta^*||_\Sigma$ ($\cB_2^d$ denote the unit ball for the canonical euclidean distance on $\R^d$)

\begin{Lemma} \label{theta}
There is $u \in \mathcal{B}_2^d$ so that, for at least 8/10 of the $k \in \{1, ..., K'\}$

$$ \braket{Z'_k, u} \geq \theta_1  $$

with $\theta_1 := 99/100||\beta_c-\beta^*||_\Sigma-r$

Moreover, for any $u \in \cB_2^d$,  at least $8/10$ of the pruned blocks check, $\braket{Z'_i,  u} \leq r + 101/100||\beta_c-\beta^*||_\Sigma$
\end{Lemma}

Now we give the main lemma from \cite{lei2019fast}, that states that it is possible to approximate  $\max_{u\in \mathcal{B}_2}\cQ^{8/10}( \braket{Z'_k, u} )$ with exponentially high probability in polynomial time.

\begin{Lemma}[Lemma 5.2 of \cite{lei2019fast}] \label{algo}
There a universal constant $C$ such that the following holds. Suppose there is $u \in \mathcal{B}_2$ so that, for at least 8/10 of the $k$

$$ \braket{ Z'_k, u} \geq \theta >0 $$
and that, for all $k$, $ Z_k'< R$. Then, when $T \geq 2 \log(K')R^2/ \theta^2$, with probability $\geq 1-\exp(-T/C)$,  algorithm \ref{algo:almost_final} applied with $T$ and $\theta$ outputs a vector $\tilde{u} \in \cB_2^d$ so that, for at least $2/10$ blocks ,  $\braket{\tilde{u},Z'_k} \geq \theta/10$ (and returns "fail" with probability $\exp(-T/C)$ ) . Moreover, each of the $T$ iteration of \ref{algo:almost_final} costs $K\times d+\polylog(d)$ operations \\

Remark : Algorithm  \ref{algo:almost_final} always return either a vector $u \in \mathcal{B}_2^d$ so that, for at least $2/10$ of the $i$, $\braket{u,Z'_k} \geq \theta/10$ or "fail". If there is no $u$ so that for at least $2/10$ blocks  $\braket{u,Z'_k} \geq \theta/10$, then it will always return "fail"
\end{Lemma}

\vspace{0.7cm}
 \begin{algorithm}[H]\label{algo:almost_final}
\SetKwInOut{Input}{input}\SetKwInOut{Output}{output}\SetKw{Or}{or}
\SetKw{Return}{Return}
\Input{$ Z'_1, \ldots,Z'_K$, $\theta$ and $T$. }
\Output{A good descent direction or "Fail".}  
\BlankLine

$R= \max(||Z'_i||)$\\
Initialize weights $\omega_1 = (1, ...,1 )/K$\\
\For{ $t =1, ...,  T $}{
 Let $A_t$ be the $K \times d$ matrix whose $i^{th}$ row is $\sqrt{\omega_t( i )} Z_i' $ and $u_t$ be the approximate top right singular vector of $A_t \times \Sigma^{-1/2}$, computed with a  $\texttt{PowerMethod}$. \\
 Set $\sigma_i = \braket{Z_i', u_t} ^2$. \\
$\omega_{t+1}(i)=\omega_t(i)\times (1-\sigma_i/2)$\\
Normalize $a=\sum_i \omega_{t+1}(i)$, $\omega_{t+1}=\omega_{t+1}/a$\\
Compute the Bregmann projection $\omega_{t+1}=\texttt{Bregmann}(\omega_{t+1})$\\
}
\Return $\texttt{ROUND}(Z',\theta,(u_t)_t)$.
 \caption{\texttt{BregmanRegression}}
\end{algorithm}
\vspace{0.7cm}

\begin{rem}
Lemma \ref{algo} has a failure probability even if we are on the events $\cA$, $\cB$ and $\cE$: it is because Algorithm \ref{algo:almost_final} calls two random algorithms, $\texttt{PowerMethod}$, which fails with constant probability, and $\texttt{ROUND}$, which fails with exponentially low probability $\propto \exp(-c T)$ with $c$ a constant (\cite{pmlr-v23-karnin12, lei2019fast}). Algorithm \ref{algo:almost_final} can tolerate at most $0.1 T$ among $T$ mistakes in the computation of the top eigenvectors of the matrices $A_t$, and the event where more than $0.1T$ of the power methods fail happesn with probability exponentially low in $T$. The failure probability of algorithm \ref{algo:almost_final} and the algorithm itself are explained in depth in  \cite{lei2019fast}.

The computation of the Bregman projection is described in \cite{doi:10.1137/1.9781611973068.129}, and appears of course in \cite{lei2019fast}.
\end{rem}

This last lemma states that finding a direction "aligned" with most of the $Z_k'$ grants a good descent direction.

\begin{Lemma} \label{goodu}
 If for at least $2/10$ blocks ,  $\braket{u, Z'_k} \geq  \theta/10 $, then $v=\Sigma^{-1/2}u$ checks $\braket{v, \Sigma \beta_c-\beta^*} \geq \theta/10 -r -||\beta_c-\beta^*||_\Sigma/100$ (and of course $||v||_\Sigma=1$ ). 
\end{Lemma}

\begin{proof}[Proof of Theorem \ref{theo:main}]

We now have all the right tools to perform our analysis.
\begin{itemize}
    \item Whenever $ ||\beta_c-\beta^*||_{\Sigma} \geq 100 r $, then by Lemma \ref{theta}, there exists $u$ so that for at least $8/10 K'$ of the (pruned) blocks $  \braket{Z'_i, u} \geq 98/100 ||\beta_c-\beta^*||_{\Sigma} $. So algorithm \ref{algo:almost_final} with $\theta \in  [49/100\norm{\beta_c-\beta^*}_\Sigma, 98/100\norm{\beta_c-\beta^*}_\Sigma]$, and with $T \geq 6 \log(K') K \geq 6 \log(K') d  \geq 2 \log(K')R^2/ \theta^2 $ does not output "Fail" (Lemma \ref{algo}).
    
    We also recall that if there is no $u$ so that for at least $4/10$ blocks ,  $\braket{u,Z_i} \geq \theta/10$, then it will always return "Fail". Thus whenever $\theta \geq 10( 101/100\norm{\beta_c-\beta^*}+r) $, by Lemma \ref{theta} , the algorithm returns "Fail".
    
    So our binary search $\texttt{stepSize} $ returns a $\theta \in [49/100 ||\beta_c-\beta^*||_\Sigma \ , \  10( 102/100\norm{\beta_c-\beta^*}_\Sigma)] \times 2/100 \times (1/10)\times (100/102)$, in less than $\log(R/||\beta_c-\beta^*||_\Sigma) \lesssim \log(d)$ iterations. The vector $u$ returned by $\texttt{descentDirection} $ is so that $v=\Sigma^{-1/2}u$ checks $\braket{v, \Sigma \beta_c-\beta^*} \geq 2 ||\beta_c-\beta^*||_\Sigma/100$, with high probability (Lemma \ref{goodu}).
    
    So we have, if $c_1=49/100\times 1/10\times 2/100\times 100/102$  and $c_0=2/100$ $$\norm{\beta_{c+1}-\beta^*}_\Sigma \leq (1-2 c_0 c_1+c_1^2) \norm{\beta_c-\beta^*}_\Sigma \leq (1-2/100.000) \norm{\beta_c-\beta^*}_\Sigma $$

    \item Whenever $ ||\beta_c-\beta^*||_{\Sigma} \leq 100 r $, whenever $\theta \geq 10( 101/100\norm{\beta_c-\beta^*}+r) $, by Lemma \ref{theta} , the algorithm returns "Fail", so our binary search $\texttt{stepSize} $ returns a $\theta\leq  \  10( 102/100\norm{\beta_c-\beta^*}_\Sigma)\times 2/100 \times (1/10)\times (100/102)= 2/100 \norm{\beta_c-\beta^*}_\Sigma)$. We have
    $$\norm{\beta_{c+1}-\beta^*}_\Sigma \leq 102/100 \norm{\beta_c-\beta^*}_\Sigma \leq 102 r$$
\end{itemize}

\end{proof}

Once again, we recall that there is no effort made here to optimize the constants. 

\section{Experiments}
\label{sec:expe}
In this section, we present the results of some synthetic numerical experiments. Our first aim is to show that our algorithm comes with actual code and that it can be computed efficiently. This is a important feature of our approach that we want to put forward because, even though there are polynomial time algorithms (even linear time ones for the problem of mean estimation) they usually do not come with efficient code. Our second aim is to show the robustness (to heavy-tailed and outliers) properties of our algorithms as predicted by our theoretical findings in Theorem~\ref{theo:main}.
\subsection{Experiments with heavy-tailed data and outliers}
\textbf{Data generating process}. We fix the contamination level $\epsilon =|\cO|/N$. Then, we generate $(1-\epsilon)N$
"clean" input vectors $X_i$ following a multivariate Student's standard t-distribution with parameter $3$ and we generate the corresponding "clean" responses following the linear model $Y=\braket{\beta^*,X} + \sigma \xi $ where $\beta^*=[1,\ldots, 1]\in\bR^d$ and where $\xi$ also follows Student's t-distribution and is independent from the feature vector $X$, and $\sigma$ is the inverse signal to noise ration (SNR). We simulate an outliers attack by adding on the $\eps N$ remaining data an arbitrary large number ($10^9$) to some cordinates of the input vectors, or multiplying them by $10^9$. We also set some responses to $0$ and some other to $10^9$. The total number of samples is set to be $N=50 d$. We note that the sample size we choose increases with the dimension. We conduct $50$ independent simulations. 

\textbf{Metric}. We measure the parameter error in $L_2$ norm, which is also the estimation norm $\norm{.}_\Sigma$ as we take $\Sigma=Id$.

\textbf{Baselines}. As our baselines, we use the Ordinary Least Square, the Huber-loss M-estimator,
RANdom SAmple Consensus (RANSAC) and the MOM-estimator from \cite{JMLR:v17:14-273}, that we name \emph{metric MOM}. The first three are implemented in the python library sci-kit learn, and we coded the last one.

\begin{figure}[h]
    \centering
    \begin{subfigure}[b]{0.50\textwidth}
        \centering
        \includegraphics[width=\textwidth]{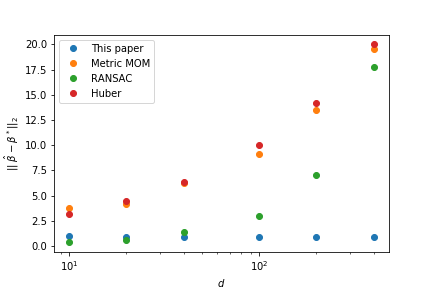}
        \caption{ Mean parameter error vs $d$ }
    \end{subfigure}%
    \begin{subfigure}[b]{0.50\textwidth}
        \centering
         \includegraphics[width=\textwidth]{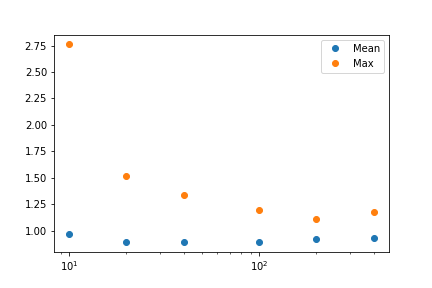} 
        \caption{ Mean and max parameter error vs  $d$}
    \end{subfigure}
    \begin{subfigure}[b]{0.50\textwidth}
        \centering
         \includegraphics[width=\textwidth]{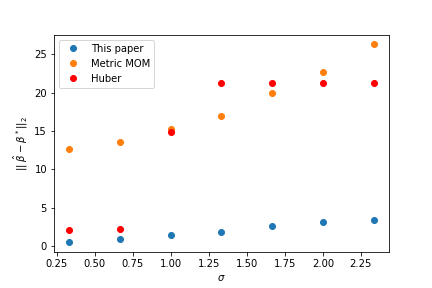} 
        \caption{ Mean parameter error vs  $\sigma$}
    \end{subfigure}
    \begin{subfigure}[b]{0.49\textwidth}
        \centering
         \includegraphics[width=\textwidth]{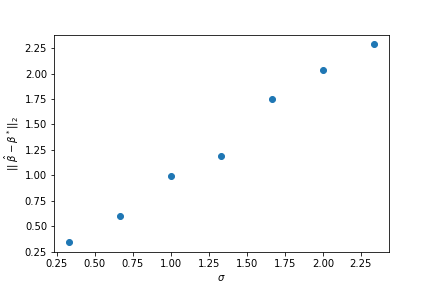} 
        \caption{ Mean parameter error vs  $\sigma$}
    \end{subfigure}

    \caption{Parameter error variations}
    \label{fig:perror}
\end{figure}

\textbf{Results}. We summarize our main findings here.
\begin{itemize}
    \item Error vs dimension $d$: We fix $\epsilon=0.005$, and we choose, for both our algorithm and the one from \cite{JMLR:v17:14-273} to take $K=d$. We do not include the OLS in our graphic because its very poor performance (due to the presence of contamination) would prevent us to compare the four others. We notice that for all the algorithms but the one presented in this paper, the prediction error grows quickly with the dimension. On the opposite, for our algorithm, the performance does not depend on the dimension. This does not come as a surprise, as the error is $\propto \sigma  K/N$, which we chose to be $d/N$, which is a fixed quantity in this setup. (Figure \ref{fig:perror}(a)). In Figure \ref{fig:perror}(b) we see a comparison between the maximum error over the 50 simulations and the mean error. We note that the maximum decreases with $d$ which seems to match the theory: since our bounds are true with probability $1-\exp(-K/c)$ (which is here equal to $1-\exp(-d/c)$), the are more frequently true as $d$ grows.

\item Error vs the inverse SNR $\sigma$: We fix $\epsilon=0.005$, $d=200$ , we still choose $K=d$ and we study how the algorithms perform for a range of SNR $\sigma$. We do not include OLS and we do not include RANSAC, because its error explodes for large $\sigma$ . We notice that our algorithm's error depends linearly on $\sigma$, which is not a surprise.

\end{itemize}

\subsection{Which choice of $K$ ?}

From a theoretical point of view, we answered the question of how one should choose the parameter $K$ in the previous section: $K$ should me at least $K\geq C_1(d \vee |\cO| \vee \log(1/\delta))$ for our algorithm to work with probability $\geq 1-\delta$, but it should not be too high because we do not want our bound $\propto K/N$ to explode. \\

\begin{figure}[h]
    \centering
    \begin{subfigure}[b]{0.5\textwidth}
        \centering
        \includegraphics[width=\textwidth]{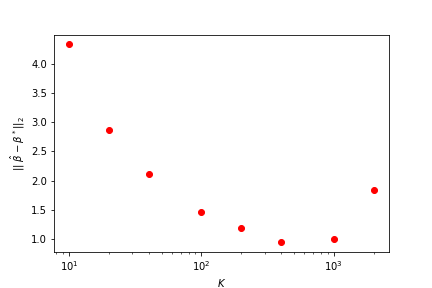}
        \caption{ Mean parameter error vs K (d=100)}
    \end{subfigure}%
    \begin{subfigure}[b]{0.5\textwidth}
        \centering
         \includegraphics[width=\textwidth]{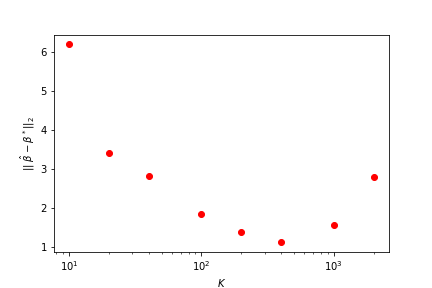} 
        \caption{ Maximum parameter error vs K (d=100)}
    \end{subfigure}
    \caption{Choice of $K$}
    \label{fig:Kchoix}
\end{figure}

 \textbf{Setup}. In Figure \ref{fig:Kchoix},  we fix the contamination level $\epsilon =|\cO|/N$ to be $0$ (there is no outlier). Then, we generate the covariates of dimension $d=100$ from a multivariate Student's t-distribution with parameter $3$ and we generate the corresponding clean responses using $y=\braket{\beta^*,x} +  \xi $ where $\beta^*=[1,\cdots, 1]$ and where $\xi$ follows Student's t-distribution and is independent from the covariates. The number of samples is set to $10 000$. We conduct $50$ independent simulations.

\textbf{Results}. The interesting thing is that we can recover a kind of trade-off from numerical experiment. It seems indeed that when $K \ll d$, our algorithm can not seize the complexity of the regression task, and that when $K \gg d$, there are not enough data per block and thus the block are "not informative enough". Those two opposite phenomenons lead to a sort of bias-variance trade-off.

\section{Conclusion}
We can outline the main benefits and limitations of our algorithm. On the practical side, the main benefit is its low computational complexity and that it comes with efficient actual code. On the theoretical side, the algorithm is robust to adversarial outliers and robust to heavy-tailed data and it achieves the subgaussian rate. It avoids the pitfall of SOS or SDPs since it uses spectral methods. This makes our algorithm both easy to understand easy to code, and that is the reason why this work comes with a simulation study unlike many other works in this literature. 

The main limitation for now is that we need to know the variance matrix $\Sigma$ of the co-variates (whereas sub optimal algorithms such as \cite{JMLR:v17:14-273} do not require knowledge of $\Sigma$). An other limitation of this work lies in the choice of $K$: we need prior knowledge on the number of outliers for our procedure to work. It might be possible to improve this with a Lepski-type procedure \cite{lerasle2019lecture}.

A final comment is that, while we choose the descent procedure from \cite{lei2019fast} for its simplicity and practical performances, the procedures from \cite{depersin2019robust} or from \cite{cherapanamjeri2019algorithms} applied with our $\tilde Z_k$'s would probably work just as well and give similar rates but may be harder to code efficiently in practice.

An interesting perspective would be to extend this work to other estimation problems such as covariance estimation, as presented in \cite{10.1145/3357713.3384329}. To do so, one would have to find an efficient way to compute $\sup_{u \in \cB_2} \sum_i \braket{u, A_i u}^2$ for any symmetric matrices $A_i$. While it is simple to compute  $\sup_{u \in \cB_2} \sum_i \braket{u, v_i}^2$ with the power method, this other problem seems harder.  We may also wonder if it is possible to adapt this kind of spectral procedure in order to recover sparse signals or, more generally, if it is possible to introduce any regularisation.  \\


\section{Proofs}

\subsection{Stochatic proofs }

We state a theorem and its direct corollary that will be useful to bound the different VC-dimensions at stake. 

 \begin{theo}[Warren, \cite{10.2307/1994937}] \label{polynomes}
 Let $P=\{P_1,...,P_m\}$ denote a set of polynomials of degree at most $\nu$ in $n$ real variables with $m > n$, then the number of sign assignments consistent for $P$ is at most $(4e\nu m/ n)^n$.
 \end{theo}
 
  We denote by $\R^n_\nu[X]$ the set of polynomials of dergree at most $\nu$ in $n$ real variables.
 \begin{coro} \label{polyVC} 
Assume that the set of functions $\cF$ can be written $\cF=\{P \in \R^n_\nu[X]\rightarrow 1_{P(x)\geq 0}, x \in \R^n\}$, then $\VC(\cF) \leq 2 n \log_2(4 e \nu)$.
 \end{coro}
Let us also recall that, if $g:\cY\to\cX$ is a function and $\cF \circ g =\{ f \circ g \ | \ f \in \cF \}$, then $\VC(\cF \circ g) \leq \VC(\cF)$.

\begin{proof}[Proof of Lemma \ref{multiplier}]

Let $\cF= \{(\textbf{x},\textbf{y}) \in \R^{(d+1)\times m} \rightarrow \mathbf{1}_{\braket{u , \sum_{i }(y_i-\braket{\beta^* , x_i})x_i}^2 \geq m^2 r^2} , u \in \Ba \}$. This is not a set of indicators of half-spaces, but $\cF$ is the composition of $g :(\textbf{x},\textbf{y}) \in R^{(d+1) \times m} \rightarrow (u \rightarrow \braket{u , \sum_{i }(y_i-\braket{\beta^* , x_i})x_i}^2-m^2 r^2 ) \in \R^d_2[X]  $ and of $\{P \in \R^d_2[X] \rightarrow 1_{P(u)\geq 0}, u \in \R^d \}$. By Corollary \ref{polyVC} , there exists an absolute constant $c$ such that $\VC(\cF) \leq c d$.

For all $u \in \Ba$,
 
 $$ \PP\left(\frac{1}{m}|\sum_{i \in B_1}(Y_i-\braket{\beta^* , X_i})\braket{u , X_i}| \geq r \right) \leq \frac{\E(\xi_1^2 \braket{u , X_1}^2)}{m r^2} \leq \frac{1}{32}. $$
 
 By Lemma \ref{main} applied with $\cF$, it follows that the following event $\cE$ has probability $\geq 1 - \exp(-K/512) $: for all $u \in \Ba$, there exist more than $3/4K$ blocks $k$ where
 
  $$  |\sum_{i \in B_k}(\tilde Y_i-\braket{a , \tilde X_i})\braket{u , \tilde X_i}| \leq  m r. $$

 \end{proof}
 
 \begin{proof}[Proof of Lemma \ref{produitscal}]
We note that, by bilinearity, it is enough to prove this result when $||u||_\Sigma=||v||_\Sigma=1.$

Let $\cG =\{(\textbf{x}_i)\in \R^{d \times m} \rightarrow \mathbf{1}_{|\sum \braket{x_i,u}\braket{x_i,v}-u\Sigma v|^2\geq c ||u||_{\Sigma}^2||v||_{\Sigma}^2} , u,v \in \R^d \}$. Once again, $\cG$ is a composition of $g :(\textbf{x},\textbf{y}) \in R^{(d+1) \times m} \rightarrow (u, v \rightarrow |\sum \braket{x_i,u}\braket{x_i,v}-u\Sigma v|^2- c ||u||_{\Sigma}^2||v||_{\Sigma}^2 ) \in \R^{2d}_4[X]  $ and of $\{P \in \R^{2d}_4[X] \rightarrow 1_{P(u)\geq 0}, u \in \R^d \}$, so there exists an absolute constant $c$ such that $\VC(\cG) \leq c d$ (Corollary \ref{polyVC}). 

Let $r_1=  6 \gamma \sqrt{\frac{1}{m}} \norm{u}_\Sigma \norm{v}_\Sigma$ . 

$$ \PP \left(|\frac{1}{m}\sum_{i \in B_1} \braket{u , X_i}\braket{v , X_i} - \braket{u , \Sigma v}| \geq  r_1 \right) \leq \frac{\E(\braket{u , X_1}^2\braket{v , X_1}^2)}{m r_1^2} \leq \frac{1}{32} $$

 because $ \E(\braket{u , X_1}^2\braket{v , X_1}^2) \leq \E(\braket{u , X_1}^4)^{1/2}\E(\braket{v , X_1}^4)^{1/2} \leq \gamma^2 \norm{u}_\Sigma^2\norm{v}_\Sigma^2$ (this is from the $L_2-L_4$ norm equivalence). We conclude with Lemma \ref{main}.

\end{proof}

\begin{proof}[Proof of Lemma \ref{init}]
We define $Z_k(\beta_c)=\sum_{j \in B_k}(Y_j-\beta X_j) \Sigma^{-1/2} X_j$.

We can write $ \norm{Z_k(\beta_c) }_2 \leq \norm{\frac{1}{m} \sum_{j \in B_k}(Y_j-\beta^* X_j) \Sigma^{-1/2} X_j }_2+ \norm{\frac{1}{m} \sum_{j \in B_k}((\beta^*-\beta_c) X_j) \Sigma^{-1/2} X_j }_2,$ we will bound those two quantities :

First $E((Y_j-\beta^* X_j) \Sigma^{-1/2} X_j)=0$, so, if $a=8\sqrt{\frac{\E(||(\xi_1 \Sigma^{-1/2} X_1||_2^2)}{m}} $
$$ \PP(||\frac{1}{m} \sum_{j \in B_1}(Y_j-\beta_c X_j) \Sigma^{-1/2} X_j ||_2\geq a )\leq \frac{\E(||(\xi_1 \Sigma^{-1/2} X_1||_2^2)}{m a^2}\leq \frac{1}{64}$$

Then, if we note $V_k=\braket{(\beta^*-\beta_c) X_j} \Sigma^{-1/2} X_j$ we notice that $\E(V_k)=\Sigma^{1/2}(\beta^*-\beta_c)$, and that 
$\E(||V_k||^2)\leq \E(\braket{\Sigma^{-1/2}X, \Sigma^{-1/2}X}^2)^{1/2} \ \E(\braket{\beta_c-\beta^*,X}^2)^{1/2}$. As $X$ checks the $L_4-L_2$ norm equivalence, $\Sigma^{-1/2}X$ checks the same equivalence, so $\E(\norm{\Sigma^{-1/2}X}_2^4)^{1/2}\leq \gamma \E(\norm{\Sigma^{-1/2}X}^2_2)= \gamma d$, and $\E(\braket{\beta_c-\beta^*,X}^2)^{1/2}=\norm{\beta_c-\beta^*}_\Sigma$, so $$\E(||\frac{1}{m}\sum_{i \in B_1} V_i||_2^2)= ||\E(V_1)||_2
^2+ \frac{1}{m}\E(|| V_i-\E(V_i)||_2^2)\leq ||\E(V_1)||_2
^2+ \frac{1}{m}\E(|| V_i||_2^2) \leq \norm{\beta_c-\beta^*}_\Sigma
^2+ \frac{1}{m} \gamma d \norm{\beta_c-\beta^*}_\Sigma^2 $$

So, as $m \geq 128 \gamma$, if $b = \sqrt{d}\norm{\beta_c-\beta^*}_\Sigma $
$$\PP(||\frac{1}{m}\sum_{i \in B_1} V_i||\geq b )\leq \frac{1}{64}$$

So the probability that one of the two bounds fails is $ \leq 1/32$. We then just use lemma \ref{main}, with the functions $\cF= \{(\textbf{x},\textbf{y})  \in \R^{(d+1)\times m} \rightarrow \mathbf{1}_{ ||\sum_{i }(y_i-\braket{\beta , x_i})x_i||^2 \geq d(r^2+\norm{\beta_c-\beta^*}_\Sigma^2) }  , \beta \in \R^d \}$. Again, we use Corollary \ref{polyVC} to state that there exists an absolute constant $c$ such that $\VC(\cG) \leq c d$.
\end{proof}

\subsection{Algorithmic proofs}

\begin{proof}[Proof of Lemma \ref{theta}]

In fact, we just know that, if we take $u=\frac{\Sigma^{1/2} (\beta_c-\beta^*)}{(||\beta_c-\beta^*||_\Sigma)}$, and  $v=\frac{ (\beta_c-\beta^*)}{(||\beta_c-\beta^*||_\Sigma)} \in B_\Sigma$ 
\begin{align}\braket{\tilde Z_i,  u} &=\sum_{i \in B_k}(\tilde Y_i-\braket{\beta^* , \tilde X_i})\braket{v , \tilde X_i}+ \sum_{i \in B_k}(\braket{\beta^*-\beta_c ,\tilde X_i})\braket{v ,\tilde X_i}
\end{align}

So for at least $9/10$ blocks, $\braket{\tilde Z_i,  u}\geq 99/100 ||\beta_c-\beta^*||_\Sigma-r := \theta_1 $. This is true for at least $9/10$ of the blocks $(\tilde Z_i)$, it is true for at least $17/19 >8/10$ of the "pruned blocks" $(Z'_i)$.\\

The same way, for any $u \in \cB_2$, we take $v = \Sigma^{-1/2} u  \in \cB_\Sigma$
\begin{align*}\braket{\tilde Z_i,  u} &=\sum_{i \in B_k}(\tilde Y_i-\braket{\beta^* , X_i})\braket{v , X_i}+ \sum_{i \in B_k}(\braket{\beta^*-\beta_c , X_i})\braket{v , X_i} \\
\leq & r +  \braket{ \beta^*-\beta_c , \Sigma v} + 1/100||\beta_c-\beta^*||_\Sigma \\
\leq & r+101/100 ||\beta_c-\beta^*||_\Sigma
\end{align*}
 for at least $9/10$ of the blocks. Again, as this is true for at least $9/10$ of the blocks, it is true for at least $17/19 >8/10$ of the "pruned blocks"

\end{proof}

\begin{proof}[Proof of Lemma \ref{goodu}]
\begin{align*}\braket{\tilde Z_i,  u} &=\sum_{i \in B_k}(\tilde Y_i-\braket{\beta^* , X_i})\braket{v , X_i}+ \sum_{i \in B_k}(\braket{\beta^*-\beta_c , X_i})\braket{v , X_i} \\
\leq & r +  \braket{ \beta^*-\beta_c , \Sigma v} + 1/100||\beta_c-\beta^*||_\Sigma
\end{align*}
for at least $9/10$ of the blocks $\tilde Z_i$.  Again, as this is true for at least $9/10$ of the blocks, it is true for at least $17/19 >8/10$ of the "pruned blocks" $Z_i'$.

Their is at least one block that checks both  $\braket{u, Z'_i} \geq  \theta/10 $ and $ \braket{u, Z'_i} \leq  r +  \braket{ \beta^*-\beta_c , \Sigma v} + 1/100||\beta_c-\beta^*||_\Sigma $  (as $2/10+17/19 >1$), so   $$\braket{ \beta^*-\beta_c , \Sigma v} \geq  \theta/10 -r -||\beta_c-\beta^*||_\Sigma/100 $$

\end{proof}

\textbf{Acknowlegements:} I would like to thank Guillaume Lecué and Matthieu Lerasle for their precious comments. I thank Yannick Guyonvarch and Fabien Perez for their help.

\begin{footnotesize}
\bibliographystyle{plain}
\bibliography{biblio}
\end{footnotesize}

\section{Appendix}
\vspace{0.7cm}
 \begin{algorithm}[H]\label{algo:Round}
\SetKwInOut{Input}{input}\SetKwInOut{Output}{output}\SetKw{Or}{or}
\SetKw{Return}{Return}
\Input{$\tilde Z_1, \ldots,\tilde Z_K$, $\theta$ and $u_1, ..., u_T$. }
\Output{u.}  
\BlankLine

\While{$\braket{Z_i,u} \leq \theta/10$ for more than $0.6K$ blocks }{
$g_j \sim \cN(0,1)$ for $j \in \{1, ..., T\}$ \\
$u=\sum_j g_j u_i/ ||\sum_j g_j u_i|| $\\
Report "Fail" and exit if more than $T$ trials have been performed }
\Return $u$.
 \caption{\texttt{Round}}
\end{algorithm}
\vspace{0.7cm}

\vspace{0.7cm}
 \begin{algorithm}[H]\label{algo:distance}
\SetKwInOut{Input}{input}\SetKwInOut{Output}{output}\SetKw{Or}{or}
\SetKw{Return}{Return}
\Input{$\tilde X_1, \tilde Y_1 \ldots, \tilde X_N, \tilde Y_N$, $\beta_c$, $K \geq |\cO|$, $T_{des}$ }
\Output{A good distance estimation, $d_t$}  
\BlankLine
Let, for $i\leq K$,  $\tilde Z_i = \frac{1}{m} \sum_{j \in B_i}(\tilde Y_j-\beta_c \tilde X_j)\Sigma^{-1/2} \tilde X_j  $\\
$Z'=\texttt{prune}(\tilde Z)$\\
$R= \max(\tilde Z'_i)$ \\
$d_{high}=R, \ d_{low}=0$\\
\For{$j \in \{1,2, ..., \lfloor\log(K)\rfloor\}$}{
$d_m=(d_{high}+d_{low})/2$\\
\If{$\texttt{BregmanRegression}( Z', d, \log( T_{des})+\log(K)K)$ returns "Fail"}{
$d_{high}\leftarrow d_m$\\
}
\Else{$d_{low}\leftarrow d_m$}
 }

\Return $d_{low}\times 2/100 \times (1/10)\times (100/102)$.
 \caption{\texttt{stepSize}}
\end{algorithm}
\vspace{0.7cm}

\vspace{0.7cm}
 \begin{algorithm}[H]\label{algo:gradient}
\SetKwInOut{Input}{input}\SetKwInOut{Output}{output}\SetKw{Or}{or}
\SetKw{Return}{Return}
\Input{$\tilde X_1, \tilde Y_1 \ldots, \tilde X_N, \tilde Y_N$, $\beta_c$, $K \geq |\cO|$, $T_{des}$, $\theta$. }
\Output{$u$.}  
\BlankLine
Let, for $i\leq K$,  $\tilde Z_i = \frac{1}{m} \sum_{j \in B_i}(\tilde Y_j-\beta_c \tilde X_j)\Sigma^{-1/2} \tilde X_j  $\\
$\texttt{prune}(\tilde Z)$\\
$u=\texttt{BregmanRegression}(\tilde Z, \theta \times 100/2 \times (10)\times (102/100), \log( T_{des})+\log(K)K)$\\
\Return $\Sigma^{-1/2} u$.
 \caption{\texttt{descentDirection}}
\end{algorithm}
\vspace{0.7cm}

\end{document}

%% file: packages_notes.tex
\usepackage{graphicx}
\usepackage{color}
\usepackage{amsmath}
\usepackage{amssymb}
\usepackage{amscd}
\usepackage{bbm}
\usepackage{tikz}
\usetikzlibrary{patterns}
\usepackage{hyperref}
\hypersetup{
    bookmarks=true,         
    colorlinks=true,       
    linkcolor=red,          
    citecolor=green,        
    filecolor=magenta,      
    urlcolor=cyan           
}

\usepackage[utf8]{inputenc}
\usepackage{amsthm}
\usepackage{bbm} 
\usepackage[linesnumbered,lined,boxed,commentsnumbered]{algorithm2e}

\usepackage{marginnote}

\newtheorem{Lemma}{Lemma}


%% file: commandes_guillaume.tex
\newcommand{\inr}[1]{\bigl< #1 \bigr>}

\newcommand{\norm}[1]{\left\|#1\right\|}%
\newcommand{\tnorm}[1]{{\left\vert\kern-0.25ex\left\vert\kern-0.25ex\left\vert #1 
    \right\vert\kern-0.25ex\right\vert\kern-0.25ex\right\vert}}

\newcommand\eps{\epsilon}

\DeclareMathOperator*{\argmin}{argmin}

\def\ds1{\textrm{1\kern-0.25emI}} 

\newcommand \E{\mathbb{E}}
\newcommand \R{\mathbb{R}}

\newcommand \cA{{\cal A}}
\newcommand \cB{{\cal B}}

\newcommand \cE{{\cal E}}
\newcommand \cF{{\cal F}}
\newcommand \cG{{\cal G}}

\newcommand \cI{{\cal I}}

\newcommand \cN{{\cal N}}
\newcommand \cO{{\cal O}}

\newcommand \cQ{{\cal Q}}

\newcommand \cX{{\cal X}}
\newcommand \cY{{\cal Y}}

\newcommand \bR{{\mathbb R}}





%% file: Regression_any_loss.bbl
\begin{thebibliography}{10}

\bibitem{MR1688610}
Noga Alon, Yossi Matias, and Mario Szegedy.
\newblock The space complexity of approximating the frequency moments.
\newblock {\em J. Comput. System Sci.}, 58(1, part 2):137--147, 1999.
\newblock Twenty-eighth Annual ACM Symposium on the Theory of Computing
  (Philadelphia, PA, 1996).

\bibitem{v008a006}
Sanjeev Arora, Elad Hazan, and Satyen Kale.
\newblock The multiplicative weights update method: a meta-algorithm and
  applications.
\newblock {\em Theory of Computing}, 8(6):121--164, 2012.

\bibitem{audibert2011}
Jean-Yves Audibert and Olivier Catoni.
\newblock Robust linear least squares regression.
\newblock {\em Ann. Statist.}, 39(5):2766--2794, 10 2011.

\bibitem{doi:10.1137/1.9781611973068.129}
Boaz Barak, Moritz Hardt, and Satyen Kale.
\newblock {\em The Uniform Hardcore Lemma via Approximate Bregman Projections},
  pages 1193--1200.

\bibitem{pmlr-v20-biggio11}
Battista Biggio, Blaine Nelson, and Pavel Laskov.
\newblock Support vector machines under adversarial label noise.
\newblock In Chun-Nan Hsu and Wee~Sun Lee, editors, {\em Proceedings of the
  Asian Conference on Machine Learning}, volume~20 of {\em Proceedings of
  Machine Learning Research}, pages 97--112, South Garden Hotels and Resorts,
  Taoyuan, Taiwain, 14--15 Nov 2011. PMLR.

\bibitem{AIHPB_2012__48_4_1148_0}
Olivier Catoni.
\newblock Challenging the empirical mean and empirical variance: A deviation
  study.
\newblock {\em Annales de l'I.H.P. Probabilit\'es et statistiques},
  48(4):1148--1185, 2012.

\bibitem{MR3052407}
Olivier Catoni.
\newblock Challenging the empirical mean and empirical variance: a deviation
  study.
\newblock {\em Ann. Inst. Henri Poincar\'{e} Probab. Stat.}, 48(4):1148--1185,
  2012.

\bibitem{MR3909640}
Yu~Cheng, Ilias Diakonikolas, and Rong Ge.
\newblock High-dimensional robust mean estimation in nearly-linear time.
\newblock In {\em Proceedings of the {T}hirtieth {A}nnual {ACM}-{SIAM}
  {S}ymposium on {D}iscrete {A}lgorithms}, pages 2755--2771. SIAM,
  Philadelphia, PA, 2019.

\bibitem{Bartlett19}
Yeshwanth Cherapanamjeri, Nicolas Flammarion, and Peter~L. Bartlett.
\newblock Fast mean estimation with sub-gaussian rates, 2019.

\bibitem{cherapanamjeri2019algorithms}
Yeshwanth Cherapanamjeri, Samuel~B. Hopkins, Tarun Kathuria, Prasad
  Raghavendra, and Nilesh Tripuraneni.
\newblock Algorithms for heavy-tailed statistics: Regression, covariance
  estimation, and beyond, 2019.

\bibitem{10.1145/3357713.3384329}
Yeshwanth Cherapanamjeri, Samuel~B. Hopkins, Tarun Kathuria, Prasad
  Raghavendra, and Nilesh Tripuraneni.
\newblock Algorithms for heavy-tailed statistics: Regression, covariance
  estimation, and beyond.
\newblock In {\em Proceedings of the 52nd Annual ACM SIGACT Symposium on Theory
  of Computing}, STOC 2020, page 601–609, New York, NY, USA, 2020.
  Association for Computing Machinery.

\bibitem{chinot2018statistical}
Geoffrey Chinot, Lecu{\'e} Guillaume, and Lerasle Matthieu.
\newblock Statistical learning with lipschitz and convex loss functions.
\newblock {\em arXiv preprint arXiv:1810.01090}, 2018.

\bibitem{depersin2020robust}
Jules Depersin.
\newblock Robust subgaussian estimation with vc-dimension, 2020.

\bibitem{depersin2019robust}
Jules Depersin and Guillaume Lecué.
\newblock Robust subgaussian estimation of a mean vector in nearly linear time,
  2019.

\bibitem{MR3576558}
Luc Devroye, Matthieu Lerasle, Gabor Lugosi, and Roberto~I. Oliveira.
\newblock Sub-{G}aussian mean estimators.
\newblock {\em Ann. Statist.}, 44(6):2695--2725, 2016.

\bibitem{MR3631028}
Ilias Diakonikolas, Gautam Kamath, Daniel~M. Kane, Jerry Li, Ankur Moitra, and
  Alistair Stewart.
\newblock Robust estimators in high dimensions without the computational
  intractability.
\newblock In {\em 57th {A}nnual {IEEE} {S}ymposium on {F}oundations of
  {C}omputer {S}cience---{FOCS} 2016}, pages 655--664. IEEE Computer Soc., Los
  Alamitos, CA, 2016.

\bibitem{MR3909639}
Ilias Diakonikolas, Weihao Kong, and Alistair Stewart.
\newblock Efficient algorithms and lower bounds for robust linear regression.
\newblock In {\em Proceedings of the {T}hirtieth {A}nnual {ACM}-{SIAM}
  {S}ymposium on {D}iscrete {A}lgorithms}, pages 2745--2754. SIAM,
  Philadelphia, PA, 2019.

\bibitem{guillaume2017learning}
Lecué Guillaume and Lerasle Matthieu.
\newblock Learning from mom's principles: Le cam's approach, 2017.

\bibitem{MR0301858}
Frank~R. Hampel.
\newblock A general qualitative definition of robustness.
\newblock {\em Ann. Math. Statist.}, 42:1887--1896, 1971.

\bibitem{MR0359096}
Frank~R. Hampel.
\newblock Robust estimation: a condensed partial survey.
\newblock {\em Z. Wahrscheinlichkeitstheorie und Verw. Gebiete}, 27:87--104,
  1973.

\bibitem{hopkins2018sub}
Samuel~B Hopkins.
\newblock Sub-gaussian mean estimation in polynomial time.
\newblock {\em arXiv preprint arXiv:1809.07425}, 2018.

\bibitem{JMLR:v17:14-273}
Daniel Hsu and Sivan Sabato.
\newblock Loss minimization and parameter estimation with heavy tails.
\newblock {\em Journal of Machine Learning Research}, 17(18):1--40, 2016.

\bibitem{MR0161415}
Peter~J. Huber.
\newblock Robust estimation of a location parameter.
\newblock {\em Ann. Math. Statist.}, 35:73--101, 1964.

\bibitem{MR2488795}
Peter~J. Huber and Elvezio~M. Ronchetti.
\newblock {\em Robust statistics}.
\newblock Wiley Series in Probability and Statistics. John Wiley \& Sons, Inc.,
  Hoboken, NJ, second edition, 2009.

\bibitem{MR855970}
Mark~R. Jerrum, Leslie~G. Valiant, and Vijay~V. Vazirani.
\newblock Random generation of combinatorial structures from a uniform
  distribution.
\newblock {\em Theoret. Comput. Sci.}, 43(2-3):169--188, 1986.

\bibitem{karnin2011furthest}
Zohar Karnin, Edo Liberty, Shachar Lovett, Roy Schwartz, and Omri Weinstein.
\newblock On the furthest hyperplane problem and maximal margin clustering,
  2011.

\bibitem{pmlr-v23-karnin12}
Zohar Karnin, Edo Liberty, Shachar Lovett, Roy Schwartz, and Omri Weinstein.
\newblock Unsupervised svms: On the complexity of the furthest hyperplane
  problem.
\newblock In Shie Mannor, Nathan Srebro, and Robert~C. Williamson, editors,
  {\em Proceedings of the 25th Annual Conference on Learning Theory}, volume~23
  of {\em Proceedings of Machine Learning Research}, pages 2.1--2.17,
  Edinburgh, Scotland, 25--27 Jun 2012. PMLR.

\bibitem{Kol11}
V.~Koltchinskii.
\newblock {\em Oracle {I}nequalities in {E}mpirical {R}isk {M}inimization and
  {S}parse {R}ecovery {P}roblems}.
\newblock Springer, Berlin, 2011.

\bibitem{KoM13}
Vladimir Koltchinskii and Shahar Mendelson.
\newblock Bounding the smallest singular value of a random matrix without
  concentration.
\newblock {\em Int. Math. Res. Notices}, to appear.
\newblock arXiv:1312.3580.

\bibitem{LeM14}
G.~Lecu{\'e} and S.~Mendelson.
\newblock Sparse recovery under weak moment assumptions.
\newblock {\em J. Eur. Math. Soc.}, to appear.
\newblock ArXiv:1401.2188.

\bibitem{lecue2013learning}
Guillaume Lecu{\'e} and Shahar Mendelson.
\newblock Learning subgaussian classes: Upper and minimax bounds.
\newblock {\em arXiv preprint arXiv:1305.4825}, 2013.

\bibitem{MR3474824}
Guillaume Lecu\'{e} and Shahar Mendelson.
\newblock Performance of empirical risk minimization in linear aggregation.
\newblock {\em Bernoulli}, 22(3):1520--1534, 2016.

\bibitem{lecu2017robust}
Guillaume Lecué and Matthieu Lerasle.
\newblock Robust machine learning by median-of-means : theory and practice,
  2017.

\bibitem{lecu2016regularization}
Guillaume Lecué and Shahar Mendelson.
\newblock Regularization and the small-ball method i: sparse recovery, 2016.

\bibitem{lei2019fast}
Zhixian Lei, Kyle Luh, Prayaag Venkat, and Fred Zhang.
\newblock A fast spectral algorithm for mean estimation with sub-gaussian
  rates, 2019.

\bibitem{lerasle2019lecture}
Matthieu Lerasle.
\newblock Lecture notes: Selected topics on robust statistical learning theory,
  2019.

\bibitem{lugosi2016risk}
Gabor Lugosi and Shahar Mendelson.
\newblock Risk minimization by median-of-means tournaments, 2016.

\bibitem{lugosi2019mean}
Gabor Lugosi and Shahar Mendelson.
\newblock Mean estimation and regression under heavy-tailed distributions--a
  survey, 2019.

\bibitem{lugosi2019sub}
G{\'a}bor Lugosi, Shahar Mendelson, et~al.
\newblock Sub-gaussian estimators of the mean of a random vector.
\newblock {\em The Annals of Statistics}, 47(2):783--794, 2019.

\bibitem{LMSL}
Z.~Szabo M.~Lerasle, T.~Matthieu and G.~Lecué.
\newblock Monk – outliers-robust mean embedding estimation by
  median-of-means.
\newblock Technical report, CNRS, University of Paris 11, Ecole Polytechnique
  and CREST, 2017.

\bibitem{MR2319879}
Pascal Massart.
\newblock {\em Concentration inequalities and model selection}, volume 1896 of
  {\em Lecture Notes in Mathematics}.
\newblock Springer, Berlin, 2007.
\newblock Lectures from the 33rd Summer School on Probability Theory held in
  Saint-Flour, July 6--23, 2003, With a foreword by Jean Picard.

\bibitem{minsker2015geometric}
Stanislav Minsker.
\newblock Geometric median and robust estimation in banach spaces.
\newblock {\em Bernoulli}, 21(4):2308--2335, 2015.

\bibitem{MR702836}
A.~S. Nemirovsky and D.~B.~and Yudin.
\newblock {\em Problem complexity and method efficiency in optimization}.
\newblock A Wiley-Interscience Publication. John Wiley \& Sons, Inc., New York,
  1983.
\newblock Translated from the Russian and with a preface by E. R. Dawson,
  Wiley-Interscience Series in Discrete Mathematics.

\bibitem{MR3568047}
Roberto~Imbuzeiro Oliveira.
\newblock The lower tail of random quadratic forms with applications to
  ordinary least squares.
\newblock {\em Probab. Theory Related Fields}, 166(3-4):1175--1194, 2016.

\bibitem{prasad2018robust}
Adarsh Prasad, Arun~Sai Suggala, Sivaraman Balakrishnan, and Pradeep Ravikumar.
\newblock Robust estimation via robust gradient estimation, 2018.

\bibitem{MR0133937}
John~W. Tukey.
\newblock The future of data analysis.
\newblock {\em Ann. Math. Statist.}, 33:1--67, 1962.

\bibitem{MR3301300}
Sara van~de Geer and Alan Muro.
\newblock On higher order isotropy conditions and lower bounds for sparse
  quadratic forms.
\newblock {\em Electron. J. Stat.}, 8(2):3031--3061, 2014.

\bibitem{vandervaart}
Aad van~der Vaart and Jon~A. Wellner.
\newblock {\em A note on bounds for VC dimensions}, volume Volume 5 of {\em
  Collections}, pages 103--107.
\newblock Institute of Mathematical Statistics, Beachwood, Ohio, USA, 2009.

\bibitem{10.2307/1994937}
Hugh~E. Warren.
\newblock Lower bounds for approximation by nonlinear manifolds.
\newblock {\em Transactions of the American Mathematical Society},
  133(1):167--178, 1968.

\end{thebibliography}
